\pdfoutput=1

\documentclass[11pt]{article}

\newcommand{\losscomparisonTable}{{
\centering
\captionof{table}{Comparison of the normalized truncation loss ($\downarrow$) between SVD-LLM and \sysname on two randomly selected weight matrices in LLaMA-3 8B using 256 calibration data on C4 under 20\% and 60\% compression ratios. \textcolor{red}{Fail} means the algorithm raises an error during the SVD compression.}
\resizebox{0.48\textwidth}{!}{%
\begin{tabular}{c|cc|cc}
\midrule
& \multicolumn{2}{c}{\textsc{Matrix A}}    & \multicolumn{2}{c}{\textsc{Matrix B}}         \\\midrule
\textsc{Ratio}     & 20\%    & 60\%      & 20\%    & 60\%     \\ \midrule
\color[HTML]{9B9B9B}Theoretical                 & \color[HTML]{9B9B9B}0.5982        & \color[HTML]{9B9B9B}2.3251           & \color[HTML]{9B9B9B}0.7351         & \color[HTML]{9B9B9B}3.5245 \\\midrule
SVD-LLM               & \textcolor{red}{Fail}         & \textcolor{red}{Fail}          & 0.8961         & 5.9834            \\\midrule
\sysname            &  \textbf{0.5982}         & \textbf{2.3251}        &\textbf{0.7351} (\textcolor{mygreen}{$\downarrow$18\%})             & \textbf{3.5245} (\textcolor{mygreen}{$\downarrow$41\%})          \\ \midrule
\end{tabular}
}
\label{tab:loss_comparison}
}}

\newcommand{\accuracyTable}{{
\begin{table*}[t]
\centering
\captionof{table}{Performance of OPT-6.7B, LLaMA-7B, and LLaMA-3 8B compressed by \sysname and baselines under 20\% compression ratio on two language modeling datasets (measured by perplexity  (\textcolor{mygreen}{$\downarrow$})), six classification datasets (measured by both individual and average accuracy (\textcolor{mygreen}{$\uparrow$})), two generation datasets (TruthfulQA measured by BLEU score (\textcolor{mygreen}{$\uparrow$}), and GSM8K measured by Exact Match Accuracy (\textcolor{mygreen}{$\uparrow$})). The best performance is marked in bold. The relative performance gain compared to the best-performing baseline is marked in green inside bracket.
}
\resizebox{1\textwidth}{!}{
\begin{tabular}{l|cccccccccccc}
\midrule
 &\multicolumn{1}{c|}{\textsc{Method}}    & WikiText-2\textcolor{mygreen}{$\downarrow$} & \multicolumn{1}{c|}{C4\textcolor{mygreen}{$\downarrow$}} & Openb. & ARC\_e & WinoG. & HellaS.  & PIQA & MathQA & \textbf{Average\textcolor{mygreen}{$\uparrow$}}  & TruthfulQA\textcolor{mygreen}{$\uparrow$} & GSM8K\textcolor{mygreen}{$\uparrow$}      \\ \midrule
 \multirow{5}{*}{\rotatebox[origin=c]{90}{\textbf{LLaMA-7B}}}& \multicolumn{1}{c|}{\color[HTML]{9B9B9B}Original}  & {\color[HTML]{9B9B9B}5.68}     & \multicolumn{1}{c|}{\color[HTML]{9B9B9B}7.34}      & {\color[HTML]{9B9B9B} 0.34} & {\color[HTML]{9B9B9B} 0.75} & {\color[HTML]{9B9B9B} 0.70} & {\color[HTML]{9B9B9B} 0.57} & {\color[HTML]{9B9B9B} 0.79} & {\color[HTML]{9B9B9B} 0.27} & {\color[HTML]{9B9B9B} 0.57}    & {\color[HTML]{9B9B9B}0.30}  & {\color[HTML]{9B9B9B}0.09}\\ \cmidrule{2-13} 
& \multicolumn{1}{c|}{FWSVD}     & 1727                     & \multicolumn{1}{c|}{1511}                      & 0.09  & 0.11  & 0.05  & 0.08  & 0.10  & 0.05  & 0.08  & 0.00   & 0.00\\
&\multicolumn{1}{c|}{ASVD}       & 11.14                   & \multicolumn{1}{c|}{15.93}                      & 0.29  & 0.53  & 0.64  & 0.41  & 0.68  & 0.17  & 0.45  & 0.21   & 0.04\\
&\multicolumn{1}{c|}{SVD-LLM}      & 7.94                   & \multicolumn{1}{c|}{15.84}                     & 0.31  & 0.71  & 0.68  & 0.49  & 0.71  & 0.22  & 0.52  & 0.24   & 0.06\\ \cmidrule{2-13} 
&\multicolumn{1}{c|}{\sysname}  & \textbf{7.12} (\textcolor{mygreen}{$\downarrow$10\%}) & \multicolumn{1}{c|}{\textbf{10.47} (\textcolor{mygreen}{$\downarrow$34\%})}        & \textbf{0.32}  & \textbf{0.72}  & \textbf{0.70}  & \textbf{0.52}  & \textbf{0.75}  & \textbf{0.24}  & \textbf{0.54} (\textcolor{mygreen}{$\uparrow$4\%})  & \textbf{0.27} (\textcolor{mygreen}{+0.03})   & \textbf{0.07} (\textcolor{mygreen}{+0.01})\\ \midrule
 \multirow{5}{*}{\rotatebox[origin=c]{90}{\textbf{OPT-6.7B}}}& \multicolumn{1}{c|}{\color[HTML]{9B9B9B}Original}  & {\color[HTML]{9B9B9B}10.87}     & \multicolumn{1}{c|}{\color[HTML]{9B9B9B}12.50}      & {\color[HTML]{9B9B9B} 0.28} & {\color[HTML]{9B9B9B} 0.66} & {\color[HTML]{9B9B9B} 0.65} & {\color[HTML]{9B9B9B} 0.50} & {\color[HTML]{9B9B9B} 0.76} & {\color[HTML]{9B9B9B} 0.25} & {\color[HTML]{9B9B9B} 0.52}    & {\color[HTML]{9B9B9B}0.29}  & {\color[HTML]{9B9B9B}0.01}\\ \cmidrule{2-13} 
&\multicolumn{1}{c|}{FWSVD}     & 14559                     & \multicolumn{1}{c|}{17898}                    & 0.03  & 0.08  & 0.02  & 0.01  & 0.05  & 0.01  & 0.03  & 0.01   & 0.00 \\
&\multicolumn{1}{c|}{ASVD}      & 82                   & \multicolumn{1}{c|}{102}                           & 0.16  & 0.41  & 0.30  & 0.36  & 0.61  & 0.07  & 0.32  & 0.09   & 0.00\\ 
&\multicolumn{1}{c|}{SVD-LLM}   & 16.04                   & \multicolumn{1}{c|}{21.27}                      & 0.21  & 0.56  & 0.59  & 0.47  & 0.73  & 0.21  & 0.46  & 0.22   & 0.00\\ \cmidrule{2-13} 
&\multicolumn{1}{c|}{\sysname}  & \textbf{13.46} (\textcolor{mygreen}{$\downarrow$16\%})  & \multicolumn{1}{c|}{\textbf{17.72} (\textcolor{mygreen}{$\downarrow$17\%})}       & \textbf{0.25}  & \textbf{0.61}  & \textbf{0.62} & \textbf{0.49}  & \textbf{0.74}  & \textbf{0.22}  & \textbf{0.49} (\textcolor{mygreen}{$\uparrow$7\%})  & \textbf{0.24} (\textcolor{mygreen}{+0.02})  & \textbf{0.01} (\textcolor{mygreen}{+0.01})\\ \midrule
\multirow{5}{*}{\rotatebox[origin=c]{90}{\textbf{LLaMA-3 8B}}}& \multicolumn{1}{c|}{\color[HTML]{9B9B9B}Original}  & {\color[HTML]{9B9B9B}6.14}     & \multicolumn{1}{c|}{\color[HTML]{9B9B9B}9.47}      & {\color[HTML]{9B9B9B} 0.35} & {\color[HTML]{9B9B9B} 0.80} & {\color[HTML]{9B9B9B} 0.73} & {\color[HTML]{9B9B9B} 0.60} & {\color[HTML]{9B9B9B} 0.80} & {\color[HTML]{9B9B9B} 0.40} & {\color[HTML]{9B9B9B} 0.61}    & {\color[HTML]{9B9B9B}0.49}  & {\color[HTML]{9B9B9B}0.45}\\ \cmidrule{2-13} 
&\multicolumn{1}{c|}{FWSVD}     & 4782                     & \multicolumn{1}{c|}{8195}                    & 0.01  & 0.04  & 0.01  & 0.02  & 0.02  & 0.01  & 0.02  & 0.00   & 0.00\\
&\multicolumn{1}{c|}{ASVD}      & 17.55                  & \multicolumn{1}{c|}{28.41}                     & 0.20  & 0.59  & 0.61  & 0.41  & 0.69  & 0.30  & 0.47  & 0.37   & 0.28\\ 
&\multicolumn{1}{c|}{SVD-LLM}   & 11.82                   & \multicolumn{1}{c|}{20.05}                    & 0.29  & 0.77  & 0.64  & 0.51  & 0.72  & 0.30  & 0.54  & 0.45   & 0.31\\ \cmidrule{2-13} 
&\multicolumn{1}{c|}{\sysname}  & \textbf{8.01} (\textcolor{mygreen}{$\downarrow$32\%}) & \multicolumn{1}{c|}{\textbf{11.72} (\textcolor{mygreen}{$\downarrow$42\%})}       & \textbf{0.33}  & \textbf{0.79}  & \textbf{0.70}  & \textbf{0.58}  & \textbf{0.77}  & \textbf{0.36}  & \textbf{0.59} (\textcolor{mygreen}{$\uparrow$9\%})  & \textbf{0.46} (\textcolor{mygreen}{+0.01})   & \textbf{0.40} (\textcolor{mygreen}{+0.09})\\ \midrule
\end{tabular}
}
\label{tab:dataset_acc}
\end{table*}
}}

\newcommand{\largellmaccuracyTable}{{
\centering
\captionof{table}{Perplexity ($\downarrow$) on WikiText-2 and average accuracy ($\uparrow$) of six classification datasets of LLaMA-13B and LLaMA-30B under 20\% compression ratio.}
\resizebox{0.48\textwidth}{!}{%
\begin{tabular}{c|cc|cc}
\midrule
& \multicolumn{2}{c}{\textsc{LLaMA-13B}}    & \multicolumn{2}{c}{\textsc{LLaMA-30B}}         \\\midrule
\textsc{Method}     & Perplexity\textcolor{mygreen}{$\downarrow$}    & Accuracy\textcolor{mygreen}{$\uparrow$}       & Perplexity\textcolor{mygreen}{$\downarrow$}    & Accuracy\textcolor{mygreen}{$\uparrow$}    \\ \midrule
\color[HTML]{9B9B9B}Original                 & \color[HTML]{9B9B9B}5.09        & \color[HTML]{9B9B9B}0.59           & \color[HTML]{9B9B9B}4.10         & \color[HTML]{9B9B9B}0.61 \\\midrule
FWSVD               & 15.98         & 0.43           & 20.54         & 0.42             \\
ASVD                & 6.74          & 0.54           & 22.71         & 0.44             \\
SVD-LLM             & 6.61          & 0.55           & 5.63          & 0.57             \\\midrule
\sysname            &  \textbf{5.46} (\textcolor{mygreen}{$\downarrow$17\%})         & \textbf{0.56} (\textcolor{mygreen}{$\uparrow$2\%})       &\textbf{4.71} (\textcolor{mygreen}{$\downarrow$16\%})             & \textbf{0.60} (\textcolor{mygreen}{$\uparrow$5\%})          \\ \midrule
\end{tabular}
}
\label{tab:large_llm_acc}
}}

\newcommand{\sensitivityTable}{{
\centering
\captionof{table}{Perplexity ($\downarrow$) of compressed LLaMA-7B by SVD-LLM and \sysname with individual / both components under 20\% compression ratio on WikiText-2.}
\resizebox{0.48\textwidth}{!}{%
\begin{tabular}{c|c|c|c}
\midrule
SVD-LLM  & \sysname (A)  &\sysname (T)   &\sysname  \\ \midrule          
7.94                 & 7.91 (\textcolor{mygreen}{$\downarrow$1\%})            & 7.43 (\textcolor{mygreen}{$\downarrow$6\%})  & 7.12  (\textcolor{mygreen}{$\downarrow$10\%})                      \\ \midrule
\end{tabular}
}
\label{tab:sensitivity}
}}

\newcommand{\svdquantTable}{{
\centering
\captionof{table}{Perplexity ($\downarrow$) of LLaMA-7B compressed by 1-bit post-training quantization methods and \sysname on WikiText-2.}
\resizebox{0.48\textwidth}{!}{%
\begin{tabular}{c|c|c|c}
\midrule
\textsc{Method}         & \textsc{Data Type}         &\textsc{Weight Memory} & \textsc{Perplexity}\\ \midrule
PB-LLM                  & 1-bit         & 1.9 GB            & 104.83 \\
BiLLM                   & 1-bit         & 1.5 GB            & 47.67           \\\midrule
\sysname                 & 16-bit        & 1.5 GB            & 99.64 \\\midrule
\sysname                  & 2-bit         & 1.5 GB             & \textbf{14.73} (\textcolor{mygreen}{$\downarrow$69\%})        \\\midrule
\end{tabular}
}
\label{tab:svd_quant}
}}

\newcommand{\svdquanthighTable}{{
\centering
\captionof{table}{Perplexity ($\downarrow$) of LLaMA-7B compressed by GPTQ and \sysname on WikiText-2.}
\resizebox{0.48\textwidth}{!}{%
\begin{tabular}{c|c|c}
\midrule
\textsc{Method}                 &\textsc{Weight Memory} & \textsc{Perplexity}\\ \midrule
GPTQ-3bit                        & 2.8 GB            & 16.28 \\\midrule
\sysname                          & 2.8 GB            & 119 \\\midrule
\sysname + GPTQ-4bit                        & 2.8 GB             & \textbf{9.97} (\textcolor{mygreen}{$\downarrow$39\%})        \\\midrule
\end{tabular}
}
\label{tab:svd_quant_4bit}
}}

\newcommand{\svdprunepplTable}{{
\centering
\captionof{table}{Perplexity ($\downarrow$) of LLaMA-7B compressed by structured pruning methods and \sysname under various weight memory budgets on WikiText-2.}
\resizebox{0.48\textwidth}{!}{%
\begin{tabular}{c|c|c|c|c}
\midrule
& \multicolumn{4}{c}{\textsc{Perplexity ($\downarrow$) under weight memory budget}}\\ \midrule
\textsc{Method}       & 10 GB       & 9 GB      & 8 GB    & 7 GB \\ \midrule
LLM-Pruner            & 9.88        & 12.21      & 18.94     & 21.68 \\
SliceGPT              & 8.78        & 12.73      & 16.39     & 27.41           \\
BlockPruner           & 9.40         & 12.76      & 19.78     & 43.05        \\\midrule
\sysname              & \textbf{7.84} (\textcolor{mygreen}{$\downarrow$17\%})        & \textbf{8.48} (\textcolor{mygreen}{$\downarrow$34\%})      & \textbf{10.17} (\textcolor{mygreen}{$\downarrow$49\%})     & \textbf{15.62} (\textcolor{mygreen}{$\downarrow$28\%})        \\\midrule
\end{tabular}
}
\label{tab:svd_prune_ppl}
}}

\newcommand{\svdpruneaccTable}{{
\centering
\captionof{table}{Average accuracy ($\uparrow$) of LLaMA-7B compressed by structured pruning methods and \sysname under various weight memory budgets.}
\resizebox{0.48\textwidth}{!}{%
\begin{tabular}{c|c|c|c|c}
\midrule
& \multicolumn{4}{c}{\textsc{Average Accuracy ($\uparrow$) under weight memory budget}}\\ \midrule
\textsc{Method}       & 10 GB       & 9 GB      & 8 GB    & 7 GB \\ \midrule
LLM-Pruner            & 0.49	& 0.47	& 0.35	& 0.31 \\
SliceGPT              & 0.51	& 0.46	& 0.38	& 0.29           \\
BlockPruner           & 0.48	& 0.46	& 0.33	& 0.20        \\\midrule
\sysname              & \textbf{0.52} (\textcolor{mygreen}{$\uparrow$2\%})        & \textbf{0.50} (\textcolor{mygreen}{$\uparrow$6\%})      & \textbf{0.42} (\textcolor{mygreen}{$\uparrow$11\%})     & \textbf{0.35} (\textcolor{mygreen}{$\uparrow$13\%})        \\\midrule
\end{tabular}
}
\vspace{2mm}
\label{tab:svd_prune_acc}
}}

\usepackage[final]{acl}

\usepackage{times}
\usepackage{latexsym}
\usepackage{amsmath}
\usepackage[T1]{fontenc}

\usepackage[utf8]{inputenc}

\usepackage{microtype}

\usepackage{inconsolata}

\usepackage{graphicx}

\usepackage{graphicx}
\usepackage{caption}
\usepackage{xspace}
\usepackage{pifont}
\usepackage{makecell}
\usepackage{xcolor}

\usepackage{subcaption}
\usepackage{booktabs} 
\usepackage{wrapfig}

\usepackage[utf8]{inputenc} 
\usepackage[T1]{fontenc}    
\usepackage{hyperref}       
\usepackage{url}            
\usepackage{amsfonts}       
\usepackage{nicefrac}       
\usepackage{xcolor}         

\usepackage{amsmath}
\usepackage{amssymb}
\usepackage{mathtools}
\usepackage{amsthm}
\usepackage{multirow}
\usepackage{graphicx}
\usepackage{lscape}
\usepackage{algorithm}
\usepackage{algpseudocode}
\usepackage[capitalize,noabbrev]{cleveref}

\theoremstyle{plain}
\newtheorem{theorem}{Theorem}[section]

\theoremstyle{definition}

\theoremstyle{remark}

\newcommand{\sysname}{\texttt{SVD-LLM} \texttt{V2}\xspace}

\usepackage[textsize=tiny]{todonotes}

\title{SVD-LLM V2: Optimizing Singular Value Truncation for \\ Large Language Model Compression}
\definecolor{mygreen}{HTML}{009901}
\definecolor{myred}{HTML}{A52A2A}

\author{Xin Wang \qquad Samiul Alam \qquad Zhongwei Wan \qquad Hui Shen \qquad Mi Zhang\\       The Ohio State University \\\texttt{\{wang.15980, alam.140, wan.512, shen.1780, mizhang.1\}@osu.edu}
\\ \url{https://github.com/AIoT-MLSys-Lab/SVD-LLM}}

\begin{document}
\maketitle
\begin{abstract}
Despite significant advancements, the practical deployment of Large Language Models (LLMs) is often hampered by their immense sizes, highlighting the need for effective compression techniques. 
Singular Value Decomposition (SVD) is a promising LLM compression technique. 
However, existing SVD-based compression methods fall short in reducing truncation losses, leading to less competitive performance in compressed models. 
In this work, we introduce \sysname, a SVD-based LLM compression method that optimizes singular value truncation in SVD compression with two techniques.
First, \sysname proposes to use theoretical truncation loss of weight matrices to assign a unique compression ratio to each weight matrix at different layers to accommodate weight redundancy heterogeneity.
%
Second, \sysname proposes loss-optimized weight truncation to ensure that the truncated singular values result in a lower and more stable truncation loss in practice. 
We evaluate \sysname on ten datasets and five LLMs at various scales. 
Our results show \sysname outperforms state-of-the-art SVD-based LLM compression methods. 
Our code is available at \url{https://github.com/AIoT-MLSys-Lab/SVD-LLM}.
\end{abstract}
\section{Introduction}
\label{sec:introduction}
Despite the outstanding performance Large Language Models (LLMs) exhibit in various tasks~\citep{DBLP:journals/corr/abs-2303-18223, DBLP:journals/corr/abs-2306-02781,DBLP:journals/corr/abs-2406-13035, DBLP:journals/corr/abs-2409-09808, wan2025meda}, the significant resources consumed limit their widespread accessibility~\citep{DBLP:journals/tmlr/Wan0LA0LQYZZC024, DBLP:journals/internet/WangWHZAZK24,zhou2024survey}. Model compression~\citep{DBLP:journals/corr/abs-2308-07633, shen2025efficient} is one effective approach to reduce resource consumption. To avoid resource-intensive retraining, LLM compression is often conducted in a post-training manner. Techniques such as LLM quantization~\citep{DBLP:conf/iclr/YuanSD24, DBLP:conf/icml/HuangLQLZ0M024}, unstructured pruning~\citep{DBLP:conf/icml/FrantarA23}, and structured pruning~\citep{DBLP:conf/nips/MaFW23, DBLP:conf/iclr/AshkboosCNHH24, DBLP:journals/corr/abs-2406-10594} have been proposed.

\begin{figure}[t]
\centering
\includegraphics[width=0.48\textwidth]{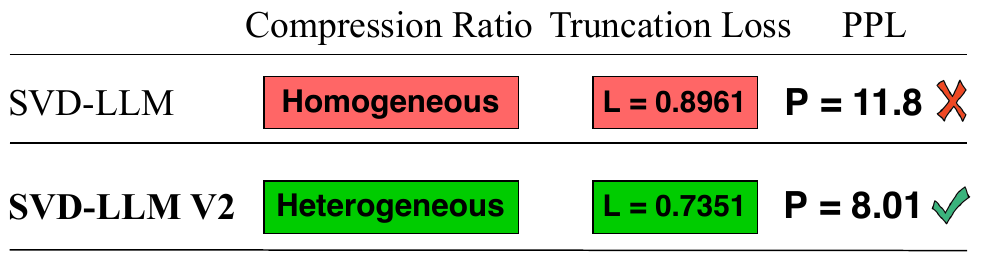}
\caption{Comparison between \sysname and SVD-LLM. We randomly select a weight matrix from LLaMA-3 8B and compare the normalized truncation loss and perplexity (PPL) under 20\% compression ratio. }
\label{fig:motivation}
\end{figure}

Low-rank approximation, such as Singular Value Decomposition (SVD) is also an effective technique for compressing LLMs. Compared with quantization and unstructured pruning, SVD compression is more hardware-friendly. 
Recently, a few SVD-based LLM compression methods
have been proposed. At a high level, these methods all focus on reducing the truncation loss during SVD compression to reserve accuracy.  
Specifically, FWSVD~\citep{DBLP:conf/iclr/HsuHCLSJ22} reduces truncation loss by estimating weight importance and preserving more important weights.
ASVD~\citep{DBLP:journals/corr/abs-2312-05821} injects a scaling matrix to reduce the truncation loss but was not able to achieve theoretical minimum truncation loss at each LLM layer. 
SVD-LLM~\citep{DBLP:journals/corr/abs-2403-07378}, on the other hand, fills this gap by proposing a whitening matrix that obtains theoretical minimum truncation loss at each LLM layer, demonstrating superior performance.

%
Despite such advantage, SVD-LLM has two limitations.
First, SVD-LLM applies a homogeneous compression ratio to all the weight matrices. This coarse-grained setup unfortunately overlooks the heterogeneity of weight redundancy across different LLM layers. 
Second, SVD-LLM utilizes Cholesky decomposition for weight truncation. However, Cholesky decomposition requires the matrix being decomposed to be positive-definite, a condition that is challenging to fulfill in practice. Moreover, Cholesky decomposition introduces numerical instability throughout its iterative process. 
As a consequence, SVD-LLM
could still lead to high truncation loss in practice.
%

In this paper, we propose \sysname, a SVD-based post-training LLM compression method that effectively addresses the two limitations of SVD-LLM.
First, to address the heterogeneity of weight redundancy across layers, \sysname uses the theoretical truncation loss of
weight matrices at each layer as the guidance to assign a unique compression
ratio to each weight matrix based on its type at different layers.
Second, \sysname substitutes the Cholesky decomposition with two rounds of SVD for weight truncation, which we prove to achieve the theoretical minimum truncation under the optimized compression ratio.
In doing so, \sysname is able to achieve better perplexity with lower truncation loss than SVD-LLM (\cref{fig:motivation}).


We evaluate \sysname on ten datasets covering various language modeling, classification, and generation tasks as well as five LLMs with various backbones and scales. Our results demonstrate the superiority of \sysname with three key findings:

\vspace{-2mm}
\begin{itemize}
    \item \sysname consistently outperforms state-of-the-art SVD-based LLM compression methods across all ten datasets and five LLMs.
    \vspace{-2mm}
    \item  \sysname outperforms state-of-the-art structured pruning-based LLM compression methods with up to 28\% lower perplexity under 7 GB memory budget. 
    When comparing to state-of-the-art 1-bit quantization-based LLM compression methods, \sysname outperforms PB-LLM and achieves 5\% lower perplexity. 
    %
    Moreover, by combining with 2-bit quantization, \sysname is able to outperform 1-bit BiLLM, demonstrating the promise of combining SVD and quantization-based methods for advancing the frontier of post-training LLM compression.
    \vspace{-2mm}
    \item LLMs compressed by \sysname achieve inference speedup on real hardware. In particular, LLMs compressed by \sysname are able to achieve a throughput speedup of up to 2.71$\times$ compared to the original LLMs on a single NVIDIA A100 GPU.
\end{itemize}

\section{Related Work}
\label{sec:related_works}
\vspace{-1mm}

\begin{figure*}[t]
\centering
\includegraphics[width=1\textwidth]{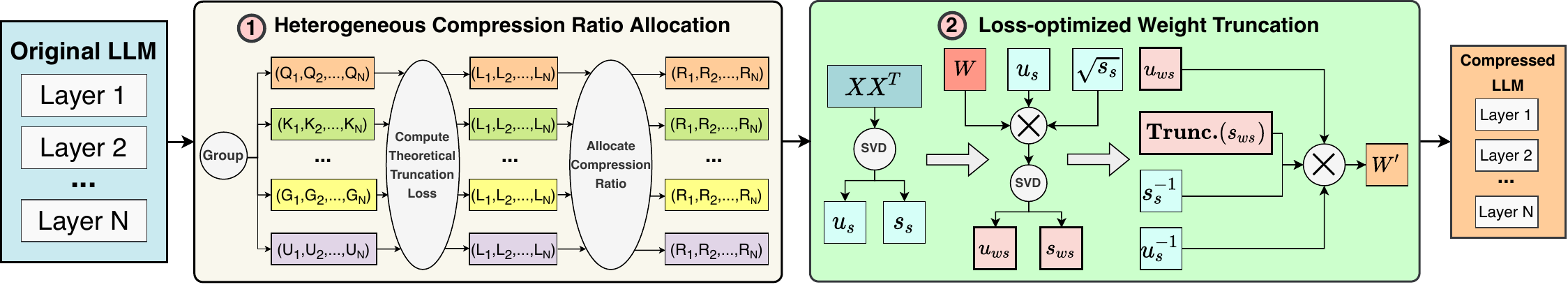}
\caption{Overview of \sysname.}
\label{fig:framework}
\end{figure*}

\subsection{Large Language Model Compression}
\label{subsc:large_language_model_compression}
Large Language Models (LLMs) typically contain billions of parameters, making traditional model compression techniques impractical due to the need for resource-intensive retraining. To address this, post-training methods that bypass retraining during compression have been developed. These methods generally fall into four categories: unstructured pruning, structured pruning, quantization, and low-rank approximation. 
Unstructured pruning~\citep{DBLP:conf/icml/FrantarA23} sets the individual weight values to zero without changing the overall architecture. However, its irregular sparsity is feasible only for speedups or memory savings on certain hardware. 
In contrast, structured pruning~\citep{DBLP:conf/nips/MaFW23,DBLP:conf/iclr/AshkboosCNHH24,DBLP:journals/corr/abs-2406-10594} removes entire channels from LLMs, simplifying hardware implementation but often suffering from accuracy degradation. 
Quantization~\citep{DBLP:journals/corr/abs-2210-17323,DBLP:journals/corr/abs-2411-07762} reduces the precision of the weight matrices for compression. However, it often fails to provide the desired inference speedups~\citep{DBLP:journals/corr/abs-2405-04532} and offers a limited range of compression options—typically between 2 to 8 bits—which hinders optimal memory utilization. Recent efforts~\citep{DBLP:conf/iclr/YuanSD24, DBLP:conf/icml/HuangLQLZ0M024} have explored 1-bit post-training quantization. Nevertheless, these approaches still suffer from accuracy drop, indicating that 1-bit quantization is still challenging in LLM compression.

\subsection{SVD for LLM Compression}
Singular Value Decomposition (SVD) reduces matrix sizes by truncating the smallest singular values. It then constructs two smaller, lower-rank matrices to approximate the original matrix~\citep{GOLUB1987317}. SVD is also feasible for LLM~\citep{DBLP:conf/iclr/HsuHCLSJ22, DBLP:journals/corr/abs-2312-05821, DBLP:journals/corr/abs-2403-07378,DBLP:journals/corr/abs-2408-09632}. To ensure better compression performance, existing post-training SVD-based LLM compression methods attempt to lower the truncation loss $L$ in the form of Frobenius norm as follows during LLM compression:

\begin{equation}
    L = ||WX-W'X||_F 
    \label{eq:t1}
\end{equation}

\noindent where $W$ is the weight matrix of the original LLM, $X$ is the activation of $W$, and $W'$ is the compressed low-ranking weight matrix. 
For example,
%
\citet{DBLP:journals/corr/abs-2312-05821} propose ASVD, which scales the weight matrix using a diagonal matrix to normalize the impact of input channels on the weights to reduce the truncation loss. 
\citet{DBLP:journals/corr/abs-2403-07378} make further advancement by whitening the input matrix to mitigate its impact on SVD truncation with the guarantee of minimal theoretical truncation loss. 
Despite these progresses, existing methods still suffer from high truncation loss in practice, leading to accuracy degradation.
\section{SVD-LLM V2}
\label{sec:methodology}



\cref{fig:framework} provides an overview of \sysname. 
Specifically, \sysname groups the weight matrices across all the layers in the original LLM by type, such as query ($Q$) and key ($K$) in attention blocks, and Gate ($G$) and Up ($U$) in MLP blocks. It then computes the theoretical truncation loss of the weight matrices and assigns a unique compression ratio to each weight matrix within each group based on the computed truncation loss.
Lastly, \sysname performs loss-optimized weight truncation to obtain the compressed LLM.
Below, we describe the details of the two main components of \sysname: (1) heterogeneous compression ratio allocation and (2) loss-optimized weight truncation. 

\subsection{Heterogeneous Compression Ratio Allocation}
\label{subsec:dynamic_ratio_allocation}
\noindent \textbf{Motivation:} Since different weight matrices in LLMs often exhibit different levels of redundancy, applying a homogeneous compression ratio to all the weight matrices would incur high truncation loss for those with low redundancy~\citep{DBLP:journals/corr/abs-2406-10594,DBLP:journals/corr/abs-2406-15786,DBLP:journals/corr/abs-2407-19126}. 
To demonstrate this, we use SVD-LLM to measure the truncation loss of the query matrix across different layers in LLaMA-3 8B on WikiText-2 dataset with 50\% compression ratio. 
As shown in~\cref{fig:layer}, the truncation loss varies at different layers. For example, the query matrix in the 27th layer has a much higher truncation loss than that of the first layer, indicating the 27th layer should be compressed under a smaller compression ratio, while a larger compression ratio should be applied to the first layer. 
However, existing SVD-based LLM compression methods either overlook this variation or require resource-intensive operations to determine the specific compression ratios, making them impractical for compressing LLMs at larger scales. 
Therefore, it is essential to develop a more efficient approach to apply different compression ratios at different weight matrices to reduce the truncation loss.

\begin{figure}[t]
    \centering
    \includegraphics[width=0.45\textwidth]{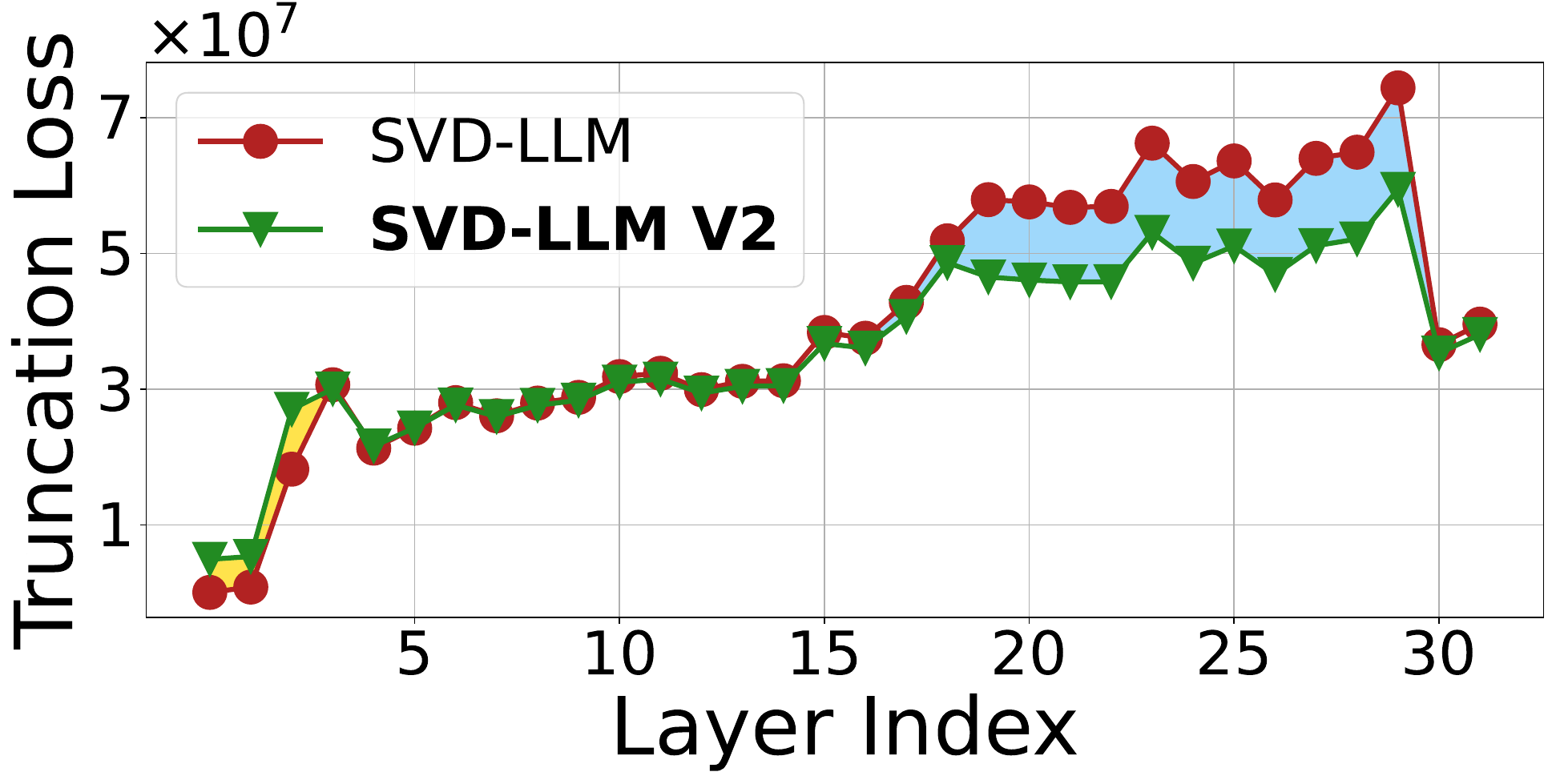}
    \caption{Comparison between SVD-LLM and \sysname on the truncation loss of the query weight matrix across different layers in LlaMA-3 8B on WikiText-2 dataset with 50\% compression ratio.}
    \label{fig:layer}
\end{figure}

\begin{algorithm}[t]
\captionsetup{font=small}
\caption{Pseudocode of Heterogeneous Compression Ratio Allocation in \sysname}
\small
\begin{algorithmic}[1] 
\Statex \textbf{Input:} $M$: Original LLM 
\Statex \hspace{3em} $x$: Input activation
\Statex \hspace{3em} $R$: Target compression ratio
\Statex \textbf{Output:} $R_d$: A list of allocated compression ratios
\Procedure{Ratio\_Allocation}{$M, S, R$} 
    \State $G \gets \boldsymbol{\operatorname{Group}}(M)$  \Comment{Group the weight by types}
    \State $R_d \gets \emptyset$ \Comment{Initialize the compression ratio list} 
        \For{$g$ \textbf{in} $G$}
        \State $L_G \gets \emptyset$ \Comment{Initialize the loss list in the group} 
            \For{$w$ \textbf{in} $g$}
                \State $L_{min} \gets \boldsymbol{\operatorname{Theoretical\_Loss}}(w, x, R)$
                \State $L_{G} \gets L_{G} \cup L_{min}$ 
            \EndFor
            \State $L_{G} \gets 1 / \boldsymbol{\operatorname{Log}}(L_{G})$ \Comment{Normalize $L_G$}
            \State $r \gets \boldsymbol{\operatorname{Len}}(L_{G}) \times R\times L_{min} / \boldsymbol{\operatorname{Sum}}(L_{G})$
            \State $R_d \gets R_d \cup r$ \Comment{Append $r$ to the list $R_d$} 
        \EndFor
    \State \Return{$R_d$}
\EndProcedure
\end{algorithmic}
\label{algo:compression_ratio}
\end{algorithm}

\noindent \textbf{Key Design:} 
The pseudocode of the heterogeneous compression ratio allocation is listed in~\cref{algo:compression_ratio}. 
Specifically, 
%
given that different types of weight matrices, such as query ($Q$) and key ($K$) in attention blocks, and Gate ($G$) and Up ($U$) in MLP blocks play different roles in an LLM, 
\sysname first groups the weight matrices across all the layers in the original LLM according to their types. 
Next, \sysname computes the theoretical minimum truncation loss of the weight matrices, i.e., $L_{min} = ||C-C'||_F$, where $C$ is the original matrix of $WX$ and $C'$  is its compressed version by SVD, respectively. It then inverses and normalizes $L_{min}$ by $1/\text{Log}(L_{min})$.
Finally, given the target model compression ratio $R$, the compression ratio of each weight matrix within a group is determined as $ \text{Len}(L_{G}) \times R \times L_{min} / \text{Sum}(L_{G})$, where $L_{G}$ denotes the list of theoretical truncation loss for all matrices within the same group, $\text{Len}(L_{G})$ denotes the group size and $\text{Sum}(L_{min})$ denotes the sum of the loss within this group. 
In this way, \sysname bypasses the need to measure end-to-end perplexity to determine compression ratios, as done in ASVD and is time-consuming. Instead, it utilizes truncation loss, which is easy to obtain and thereby enhancing the efficiency of the algorithm.
As shown in~\cref{fig:layer}, with the proposed heterogeneous compression ratio allocation scheme, \sysname effectively reduces the truncation loss (the blue area) with only a small increase of several small truncation losses (the yellow area).

%
In the next section, we describe the details of the proposed loss-optimized weight truncation.  
%

\subsection{Loss-optimized Weight Truncation}
\label{subsec:loss_optimized_weight_truncation}
\noindent \textbf{Motivation:} 
After determining the specific compression ratio for each weight matrix in the LLM, the next step is to truncate the weights according to their assigned compression ratios. 
To reduce truncation loss $L= ||WX-W'X||_F$ during SVD compression, SVD-LLM first constructs the whitening matrix $S$ by applying Cholesky decomposition on $XX^T$. It then performs SVD and truncation on $WS$. Although SVD-LLM has been theoretically proven to achieve the lowest truncation loss at a given compression ratio, our empirical study shows that its actual truncation loss is frequently above the theoretical minimum. This is mainly due to the numerical instability involved in performing the Cholesky decomposition on a large-scale matrix during truncation. Moreover, the Cholesky decomposition requires $XX^T$ to be positive definite, which is often hard to satisfy.

%
To demonstrate this, we randomly select two weight matrices, $A$ and $B$, in LLaMA-3 8B and compute their truncation loss by SVD-LLM using 256 randomly selected data in the C4 dataset under compression ratios 20\% and 60\%. As shown in~\cref{tab:loss_comparison}, because $XX^T$ is not positive definite, SVD-LLM fails to compress matrix $A$. For matrix $B$, even when the compression ratio is as low as 20\%, SVD-LLM still achieves a larger truncation loss in practice than in theory, and this difference even increases with increasing compression ratio. 

SVD-based LLM compression methods such as Balco~\citep{DBLP:journals/corr/abs-2405-10616} have been proposed that utilize pooled covariance matrices to precisely estimate the feature distribution to reduce truncation loss. However, these methods cannot guarantee their theoretical optimality during SVD truncation. 
Therefore, it is necessary to design a new way to optimize the truncation loss for SVD compression. 

\losscomparisonTable
\vspace{2mm}


\begin{algorithm}[t]
\captionsetup{font=small}
\caption{Pseudocode of Loss-optimized Weight Truncation in \sysname}
\small
\begin{algorithmic}[1] 
\Statex \textbf{Input:} $W$: Original weight matrix 
\Statex \hspace{3em} $X$: Input activation
\Statex \hspace{3em} $R$: Target compression ratio
\Statex \textbf{Output:} $W'$: Compressed weight matrix
\Procedure{Weight\_Truncation}{$W, X, R$} 
    \State $S \gets XX^T$  \Comment{Construct matrix $S$ from $X$}
    \State $U_s, S_s, V_s \gets \boldsymbol{\operatorname{SVD}}(S)$ \Comment{Perform SVD on matrix $S$} 
    \State $D \gets W\times U_s \times \sqrt{S_s}$  \Comment{Construct matrix $D$}
    \State $U_{ws}, S_{ws}, V_{ws} \gets \boldsymbol{\operatorname{SVD}}(D)$ \Comment{Perform SVD on matrix $D$}
    \State $T_s \gets \boldsymbol{\operatorname{Truncate}}(S_{ws}, R)$ \Comment{Perform SVD truncation on matrix $S_{ws}$ based on compression ratio $R$}
    \State $W' \gets U_{ws}\times T_s \times S_{s} ^{-1}\times U_{s}^{-1}$ \Comment{Construc $W'$}
    \State \Return{$W'$}
\EndProcedure
\end{algorithmic}
\label{algo:weight_truncation}
\end{algorithm}

\noindent \textbf{Key Design:}
The pseudocode of the proposed loss-optimized weight truncation is provided in~\cref{algo:weight_truncation}. Different from SVD-LLM, \sysname bypasses the Cholesky decomposition, resulting in a more straightforward process with improved numerical stability. Specifically, given the input activation $X$, \sysname conducts SVD on $XX^T$ to obtain the decomposed matrices $U_s, S_s, V_s$, where $S_s$ is the diagonal matrix with singular values. It then conducts another round of SVD on $W\times U_s\times \sqrt{S_s}$ to obtain $U_{ws}, S_{ws}, V_{ws}$. The final compressed weight matrix $W'$ can be obtained via $U_{ws}\times \boldsymbol{\operatorname{Trunc.}}(S_{ws}) \times S_{s} ^{-1}\times U_{s}^{-1} $, where $\boldsymbol{\operatorname{Trunc.}}(C)$ denotes the rank-k truncation of matrix $C$ during SVD compression.

In the following, we provide a theoretical proof on why such truncation offers the same theoretical minimum truncation loss as SVD-LLM.

\begin{theorem}
\label{thm:whitening}
If $U_s, S_s, V_s$ are obtained by SVD decomposition of $XX^T$ and $U_{ws}, S_{ws}, V_{ws}$ are obtained by SVD decomposition of $W\times U_s\times \sqrt{S_s}$, the compressed weight matrix $W' = U_{ws}\times \boldsymbol{\operatorname{Trunc.}}(S_{ws}) \times V_{ws} \times \sqrt{S_{s}}^{-1} \times U_{s}^{-1}$ ensures the theoretical minimum truncation loss.
\end{theorem}

\begin{proof}
\label{proof:truncation}
Since $XX^T$ is the symmetric matrix, suppose that the singular vectors and values of input activation $X$ is $U_x, S_x, V_x$, we have $U_s = U_x$ and $\sqrt{S_s} = S_x$. Suppose $S = U_s\times \sqrt{S_s}$, thus $S^{-1} = \sqrt{S_{s}}^{-1} \times U_{s}^{-1}$, and we have:

{\small
\begin{equation}
    \begin{aligned}
        &S^{-1} X = \sqrt{S_{s}}^{-1} U_{s}^{-1} X = S_{x}^{-1} U_{x}^{-1} X\\
        &= S_{x}^{-1} U_{x}^{-1} U_x S_x V_x = Vx
    \end{aligned}
\label{formula:derive}
\end{equation}
}

\noindent Therefore, $S^{-1}\times X$ is orthogonal and $||A\times S^{-1}\times X||_F = ||S^{-1}\times X||_F$, and the final truncation loss could be derived as:

{\small
\begin{equation}
    \begin{aligned}
        &L^2 =||W X-W^{\prime} X||_F^2 \\
        &= ||WSS^{-1}X - U_{ws}\times \boldsymbol{\operatorname{Trunc.}}(S_{ws}) \times V_{ws} \times S^{-1} X||_F^2\\
        &= ||(WS - U_{ws}\times \boldsymbol{\operatorname{Trunc.}}(S_{ws}) \times V_{ws})S^{-1}X||_F^2\\
        &= ||WS - U_{ws}\times \boldsymbol{\operatorname{Trunc.}}(S_{ws}) \times V_{ws}||_F^2\\
        &= ||\boldsymbol{\operatorname{SVD}}(WS)||_F^2 = ||\boldsymbol{\operatorname{SVD}}(W\times U_x\times S_x)||_F^2\\
        & = ||\boldsymbol{\operatorname{SVD}}(W\times U_x\times S_x\times V_x)||_F^2\\
        & = ||\boldsymbol{\operatorname{SVD}}(WX)||_F^2 = L_{min} ^2
    \end{aligned}
\label{formula:loss_1}
\end{equation}
}

\noindent Therefore, the designed SVD truncation ensures the theoretical minimum truncation loss.
\end{proof}

For a better demonstration, we also implement the new loss-optimized weight truncation by \sysname on LLaMA-3-8B. As shown in~\cref{tab:loss_comparison}, \sysname achieves better numerical stability and lower truncation loss than SVD-LLM.

\section{Experiments and Analysis}
\label{sec:experiments}
\accuracyTable
\textbf{Baselines.} 
We compare \sysname against two groups of methods. (1) Three state-of-the-art SVD-based LLM compression methods: FWSVD~\citep{DBLP:conf/iclr/HsuHCLSJ22}, ASVD~\citep{DBLP:journals/corr/abs-2312-05821}, and SVD-LLM~\citep{DBLP:journals/corr/abs-2403-07378} (Section~\ref{subsec:overall_performance}). (2) Other types of LLM compression methods. These include three state-of-the-art pruning-based LLM compression methods: LLM-Pruner~\citep{DBLP:conf/nips/MaFW23}, SliceGPT~\citep{DBLP:conf/iclr/AshkboosCNHH24}, and BlockPruner~\citep{DBLP:journals/corr/abs-2406-10594}, and two state-of-the-art quantization-based LLM compression methods: PB-LLM~\citep{DBLP:conf/iclr/YuanSD24}, and BiLLM~\citep{DBLP:conf/icml/HuangLQLZ0M024} (Section~\ref{subsection:Comparison_with_other_compression_methods}).

\vspace{1mm}
\noindent\textbf{Models and Datasets.} To demonstrate the generability of our method, we evaluate the performance of \sysname on five models at various scales (LLaMA-7B, 13B, 30B, LLaMA3-8B~\citep{DBLP:journals/corr/abs-2307-09288}, OPT-6.7B~\citep{DBLP:journals/corr/abs-2205-01068}) and ten datasets including two language modeling datasets (WikiText-2~\citep{DBLP:conf/iclr/MerityX0S17} and C4~\citep{DBLP:journals/jmlr/RaffelSRLNMZLL20}), six classification datasets (OpenbookQA~\citep{DBLP:conf/emnlp/MihaylovCKS18}, WinoGrande~\citep{DBLP:conf/aaai/SakaguchiBBC20}, HellaSwag~\citep{DBLP:conf/acl/ZellersHBFC19}, Arc\_e~\citep{DBLP:journals/corr/abs-1803-05457}, PIQA~\citep{DBLP:conf/aaai/BiskZLGC20}, MathQA~\citep{DBLP:conf/naacl/AminiGLKCH19}),  and two generation datasets (TruthfulQA~\citep{DBLP:conf/acl/LinHE22} and GSM8K~\citep{DBLP:journals/corr/abs-2110-14168})  with the LM-Evaluation-Harness framework~\citep{eval-harness}.

\vspace{1mm}
\noindent\textbf{Implementation Details.} 
We randomly select $256$ WikiText-2 samples as the calibration data. To mitigate the error raised by the Choleksy decomposition in SVD-LLM due to positive definite, we followed the implementation of SVD-LLM~\citep{DBLP:journals/corr/abs-2403-07378} to add the small noise into the decomposed matrices. The compression ration in our experiments refers to the parameter reduction of LLM achieved through compression. All of the experiments are conducted on A100 GPUs.

\subsection{Performance Comparison}
\label{subsec:overall_performance}
We first compare \sysname against state-of-the-art SVD-based LLM compression methods from four aspects: (1) performance on different LLMs, (2) performance on LLMs with larger scales, (3) performance under different compression ratios, and (4) compression speed.

\vspace{1mm}
\noindent\textbf{Performance on Different LLMs.} 
We compare the performance between \sysname and the baselines on three different LLMs, including LLaMA-7B, OPT-6.7B, and LLaMA-3 8B  under 20\% compression ratio on ten datasets. As shown in~\cref{tab:dataset_acc},  \sysname consistently achieves better and more stable performance than all the SVD-based LLM compression baselines across all three LLMs and all ten datasets. 
In particular, \sysname achieves up tp 42\% perplexity reduction and 9\% accuracy
improvement with better generation ability compared to prior state-of-the-art method SVD-LLM on LLaMA-3 8B.

\vspace{1mm}
\noindent\textbf{Performance on LLMs with Larger Scales.} 
We compare the performance between \sysname and the baselines on LLaMA-13B and LLaMA-30B under 20\% compression ratio on WikiText-2 and six classification datasets. As shown in~\cref{tab:large_llm_acc}, \sysname consistently outperforms all the baselines on both 13B and 30B model sizes.

\largellmaccuracyTable

\vspace{1mm}
\noindent\textbf{Performance under Different Compression Ratios.}
We compare the performance between \sysname and the baselines on LLaMA-7B under compression ratio ranging from 20\% to 80\% on WikiText-2 and six classification datasets. As shown in~\cref{fig:ratio}, \sysname consistently outperforms all baselines, and the performance gain compared to the best-performing baseline increases as the compression ratio increases. 

\vspace{1mm}
\noindent\textbf{Compression Speed.}
Besides measuring the performance of the compressed LLMs, we also evaluate the compression speed. Specifically, we measure the A100 GPU hours used by \sysname and the baseline methods for compressing LLaMA-7B under 20\% compression ratio. 
Our results show that FWSVD takes about 6 GPU hours, ASVD takes about 5.5 GPU hours, SVD-LLM takes about 15 minutes, and \sysname takes about 18 minutes to finish the compression. 
FWSVD requires gradient calculation, thus consumes a significant amount of time for compression. 
For the other methods, the main reason for such a variation is their respective techniques for allocating compression ratios among weight matrices. SVD-LLM assigns the same compression ratio to all weight matrices, enabling the fastest operation but sacrificing accuracy. ASVD, however, determines the compression ratio by regularly evaluating the end-to-end perplexity, which slows down its compression process. In contrast, \sysname allocates the compression ratio directly from its truncation loss, making it significantly faster than ASVD.

\begin{figure}[t]
    \begin{subfigure}{0.48\textwidth}
            \centering
            \includegraphics[width=\linewidth]{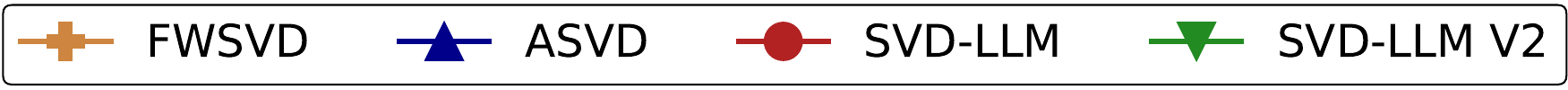}
        \end{subfigure}
        \\
         \begin{subfigure}{0.235\textwidth}
            \centering
            \includegraphics[width=\linewidth]{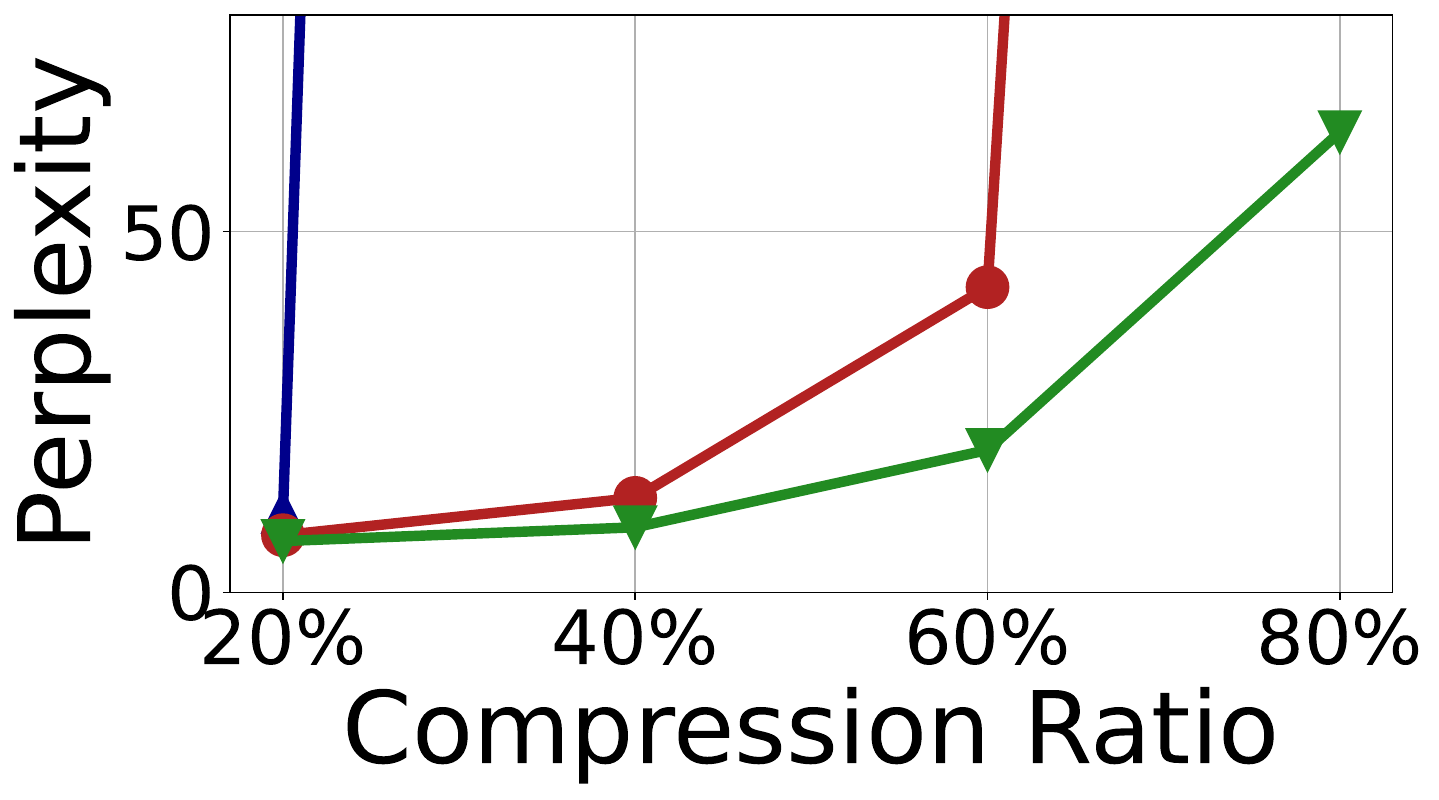}  
            \caption{Perplexity}
            \label{fig:ppl_ratio}
        \end{subfigure}
        \begin{subfigure}{0.235\textwidth}
            \centering
            \includegraphics[width=\linewidth]{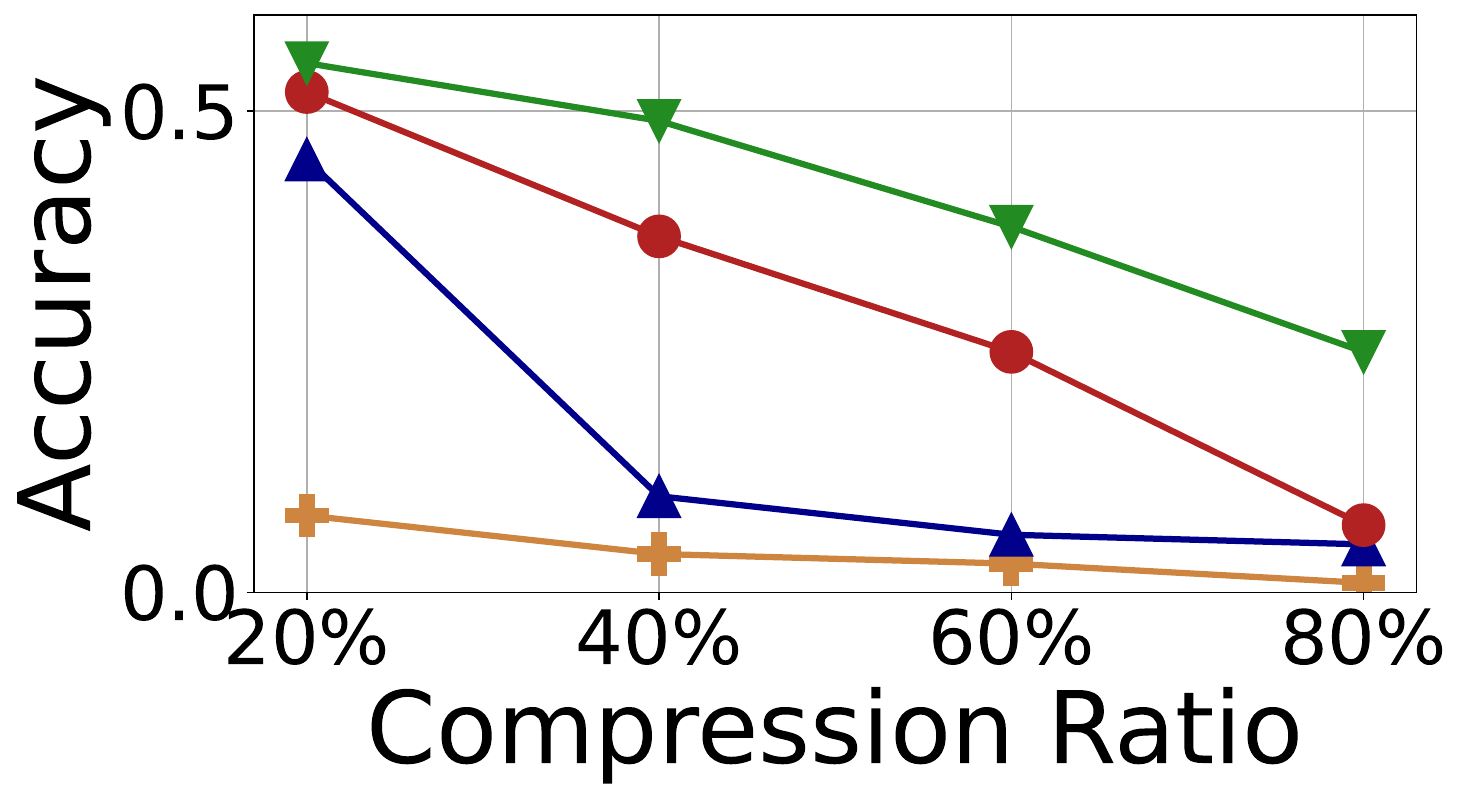}  
            \caption{Average Accuracy}
            \label{fig:acc_ratio}
        \end{subfigure}
\caption{Perplexity on WikiText-2 and average accuracy on six classification datasets of LLaMA-7B compressed by \sysname  and other SVD-based LLM compression baselines under 20\% to 80\% compression ratios. The perplexity values of FWSVD and ASVD are larger than 100, thus are not shown in the figure.}
\vspace{-5mm}
\label{fig:ratio}
\end{figure}

\subsection{Inference Speedup of \sysname }
\label{subsec:inference_speedup}
To evaluate the inference speedup of models compressed by \sysname, we measure the numbers of tokens generated per second from both the original LLaMA-7B and the model compressed by \sysname under different compression ratios on a single NVIDIA A100 GPU. 
For a fair comparison, we fix the batch size to 4, the prefill length to 1024,  and the decoding length to 256. As shown in~\cref{fig:speedup}, \sysname consistently achieves faster token generation speeds across all the compression ratios. 
More importantly, the speedup becomes more significant as the compression ratio increases, resulting in a speedup of 1.29x under 20\% compression ratio, 1.63x under 40\% compression ratio, 2.08x under 60\% compression ratio, and 2.71x under 80\% compression ratio. 

\begin{figure}[t]
    \centering
    \includegraphics[width=0.4\textwidth]{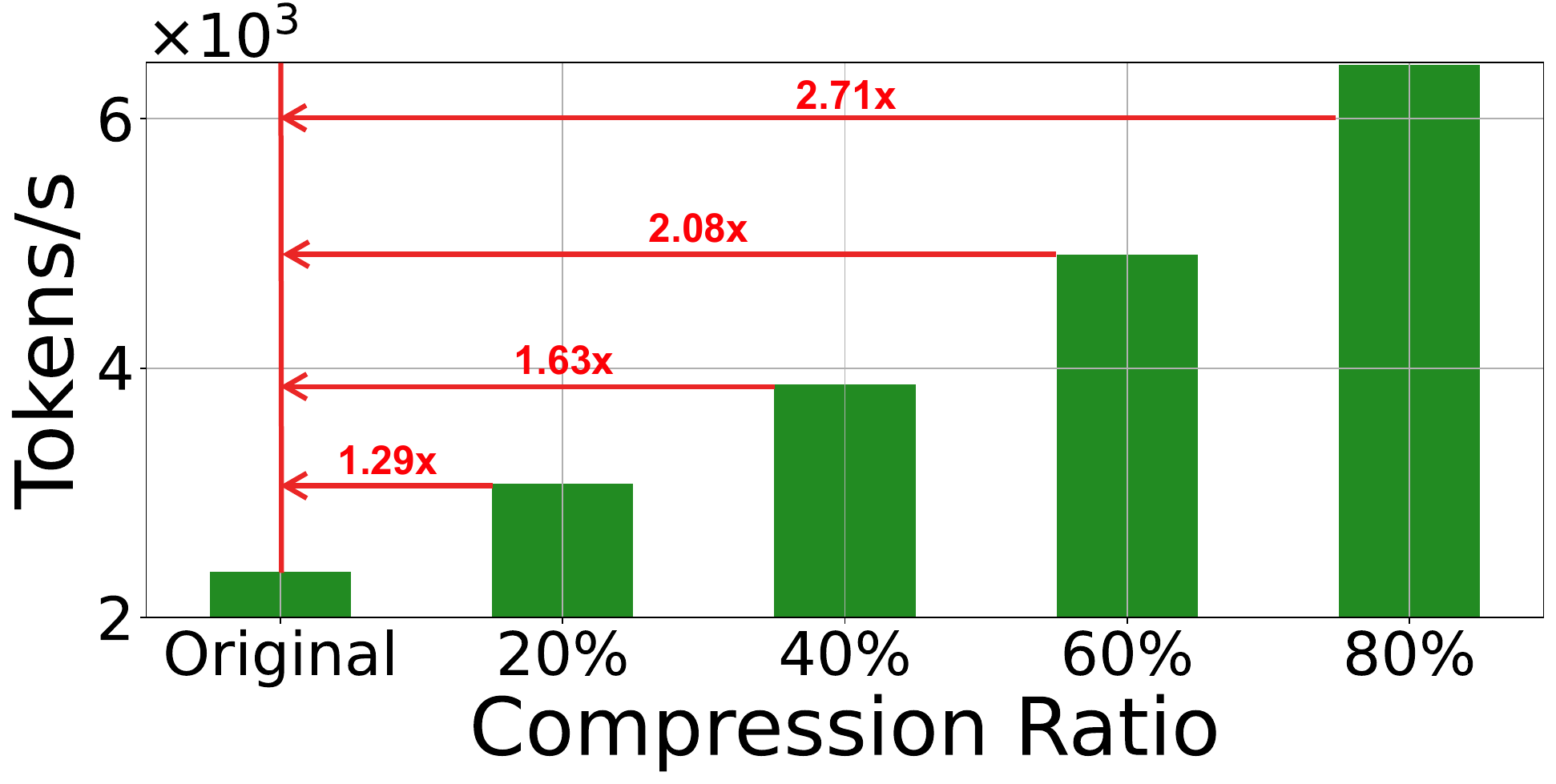}
    \caption{Throughput (Tokens/s) achieved by original LLaMA-7B and its compressed version by \sysname under different compression ratios on a single NVIDIA A100 GPU. We fix the batch size to 4, prefill length to 1024,  and decoding length to 256.  The speedup over the original LLM is marked in red.}
    \label{fig:speedup}
\end{figure}

\vspace{2mm}
\sensitivityTable

\subsection{Ablation Study}
\label{subsec:ablation_study}
\sysname has two key components, both of which optimize the truncation loss. 
In our ablation study, we first evaluate the individual contribution of each of the two components to the compression performance. 
Next, since both components fully utilize the whitening matrix $S$, which is calculated with a randomly selected calibration set, we evaluate the impacts of different calibration data on the performance of \sysname.

\vspace{1mm}
\noindent\textbf{Component Sensitivity Study.} We first evaluate the individual contribution of the two components (i.e., heterogeneous compression ratio allocation and loss-optimized weight truncation) of \sysname. 
Let \sysname (A) denote the version of \sysname with heterogeneous compression ratio allocation only; and \sysname (T) denote the version of \sysname with loss-optimized weight truncation only.
%
The results are shown in~\cref{tab:sensitivity}. We have two observations. 
(1) Both \sysname (A) and \sysname (T) outperform SVD-LLM, demonstrating the effectiveness of each of these two components alone in achieving superior compression performance.
(2) \sysname outperforms \sysname (A) and \sysname (T), demonstrating the necessity of having both components.

\vspace{1mm}
\noindent\textbf{Impact of Calibration Set.} Next, we examine the impact of the calibration set on the compression performance of \sysname. Specifically, we measure the changes in perplexity of LLaMA-7B compressed by \sysname under 20\% compression ratio on WikiText-2 when using the default calibration set but with various numbers of data and sampling seeds.
As shown in~\cref{fig:calibration}, the changes in both data number and sampling seed in the calibration set incur no more than 1\% fluctuation in the final performance, demonstrating that \sysname is not sensitive to the calibration set.

\begin{figure}[t]
         \begin{subfigure}{0.235\textwidth}
            \centering
            \includegraphics[width=\linewidth]{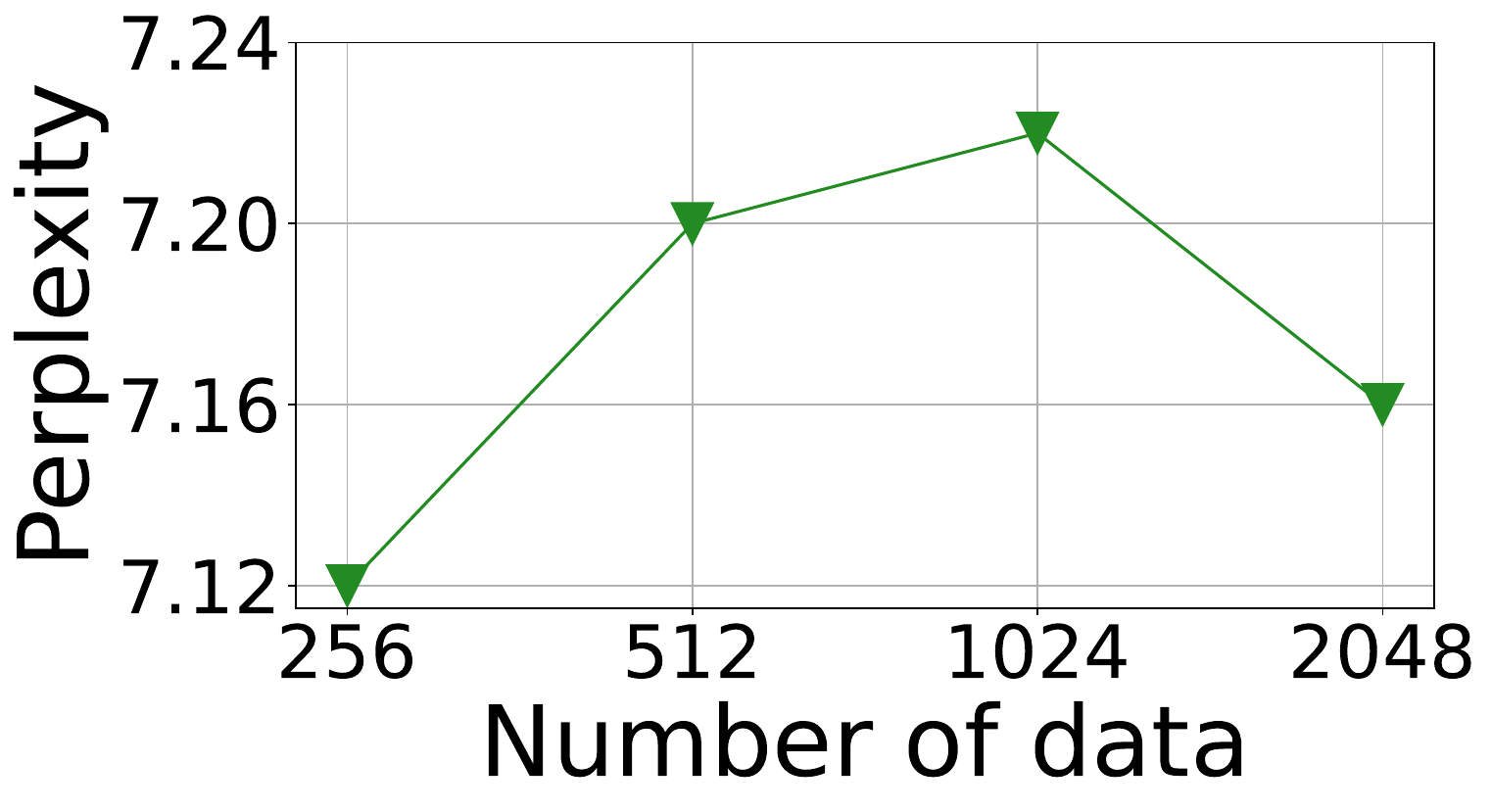}  
            \caption{Various Data Number}
            \label{fig:cali_number}
        \end{subfigure}
        \begin{subfigure}{0.24\textwidth}
            \centering
            \includegraphics[width=\linewidth]{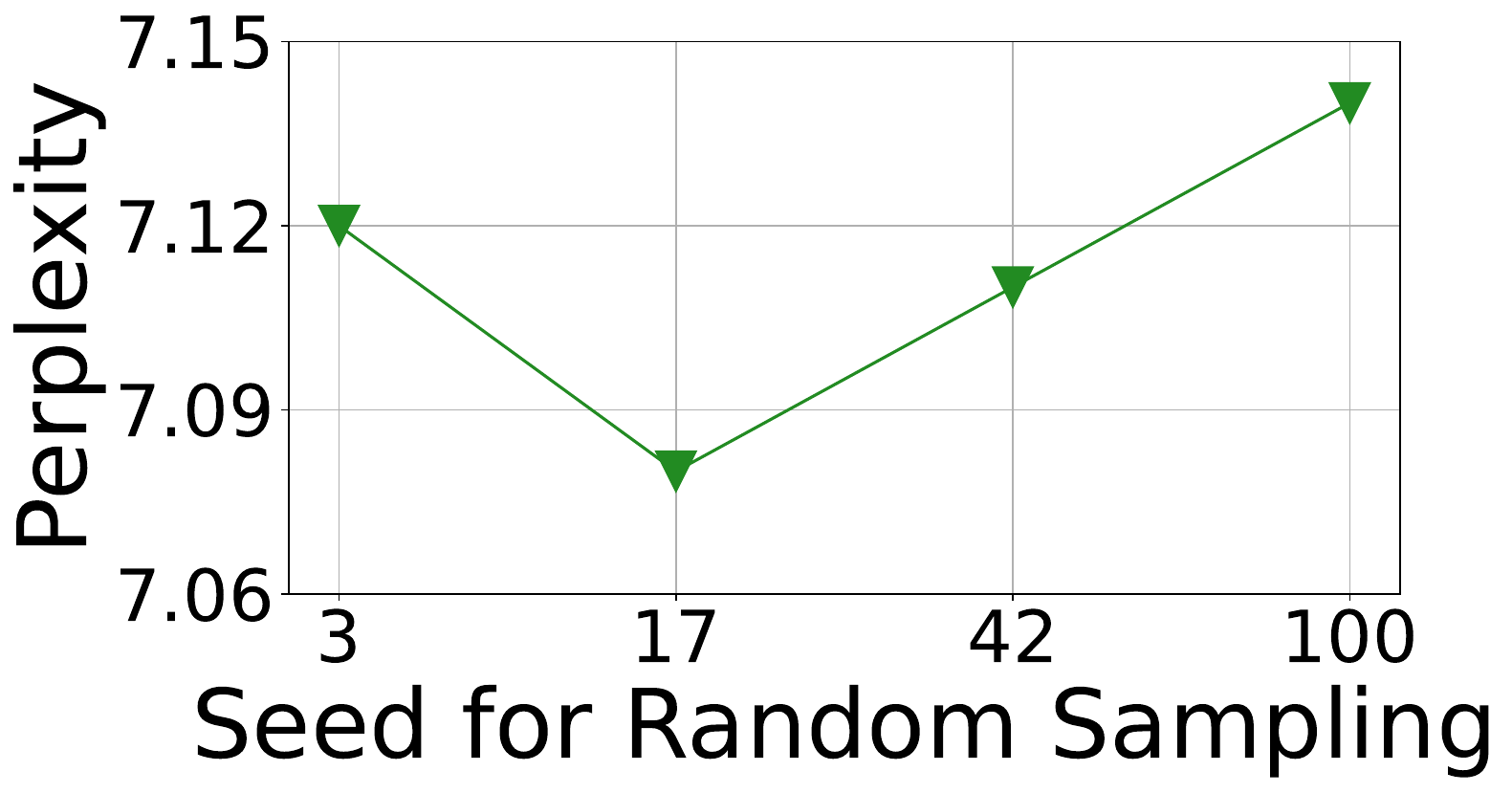}  
            \caption{Various Samping Seed}
            \label{fig:cali_seed}
        \end{subfigure}
\caption{Perplexity of LLaMA-7B under 20\% compression ratio using calibration data sampled with different numbers or seeds from WikiText-2.}
\label{fig:calibration}
\end{figure}

\subsection{Comparison with Other Types of Post-training LLM Compression Methods} 
\sysname is orthogonal to other post-training LLM compression methods, including quantization and pruning. In this experiment, we compare the performance of \sysname with state-of-the-art structured pruning-based and quantization-based LLM compression methods. 

\svdprunepplTable

\svdpruneaccTable

\noindent\textbf{Comparison with Structured Pruning.}
First, we compare \sysname with three state-of-the-art post-training structured pruning-based LLM compression methods: LLM-Pruner~\citep{DBLP:conf/nips/MaFW23}, SliceGPT~\citep{DBLP:conf/iclr/AshkboosCNHH24}, and BlockPruner~\citep{DBLP:journals/corr/abs-2406-10594} under various weight memory budgets, ranging from 10 GB to 7 GB. 
The perplexity results are shown in~\cref{tab:svd_prune_ppl} and the average accuracy results are shown in~\cref{tab:svd_prune_acc}. 
As shown, \sysname outperforms all three state-of-the-art structured pruning-based LLM compression methods. In particular, under 7 GB memory budget, \sysname achieves 28\% reduction in perplexity and 13\% higher average accuracy.

\vspace{1mm}
\noindent\textbf{Comparison with Quantization.} 
Next, we compare \sysname with post-training quantization-based LLM compression methods. 
We first compare \sysname with GPTQ~\citep{DBLP:journals/corr/abs-2210-17323} under 3-bit quantization. As shown in~\cref{tab:svd_quant_4bit}, while \sysname achieves worse perplexity compared to GPTQ under 3-bit memory budget, 
combining \sysname (30\% compression ratio) with GPTQ-4-bit achieves superior perplexity compared to GPTQ-3-bit under the same memory budget. 
In other words, we find that under the same memory budget, by first compressing the original 16-bit LLM with \sysname at 30\% compression ratio, then quantizing the compressed LLM to 4-bit using GPTQ, we are able to achieve better perplexity compared to directly quantizing the original LLM to 3-bit.  
Finally, we compare \sysname with  two state-of-the-art post-training quantization-based LLM compression methods: BiLLM~\citep{DBLP:conf/icml/HuangLQLZ0M024} and PB-LLM~\citep{DBLP:conf/iclr/YuanSD24}, which push the frontier to 1-bit quantization.
%
The results are shown in Table~\ref{tab:svd_quant}. 
We have two observations: (1) Without combining with quantization techniques, \sysname (16-bit) outperforms PB-LLM with 5\% lower perplexity.
%
%
(2) By combining with quantization techniques, \sysname (2-bit) outperforms state-of-the-art 1-bit post-training quantization method BiLLM. In particular, \sysname (2-bit) achieves 69\% lower perplexity than BiLLM, showing the promise of combining SVD-based and quantization-based compression methods for pushing the frontier of post-training LLM compression forward.

\vspace{0mm}
\svdquanthighTable

\vspace{0mm}
\svdquantTable

\label{subsection:Comparison_with_other_compression_methods}
\vspace{-1mm}
\section{Conclusion}
\label{sec:conclusion}

\vspace{-2mm}
In this paper, we present \sysname, a SVD-based post-training LLM compression method. \sysname addresses the limitation of existing methods about high truncation loss during compression. Specifically, \sysname first employs a heterogeneous compression ratio allocation strategy to effectively balance truncation loss across different weight matrices of the LLM. It further introduces a loss-optimized weight truncation to ensure a lower and more stable truncation loss. 
Our evaluation results demonstrate the superiority of \sysname over state-of-the-art SVD-based post-training LLM compression methods.

\vspace{-1mm}
\section{Limitations}
\vspace{-1mm}

While \sysname outperforms existing SVD-based LLM compression methods, there is still space for further improvement. For example, under 90\% compression ratio, there is a small performance gap compared with state-of-the-art quantization methods.
We aim to fill this gap in future.

\vspace{-1mm}
\section{Acknowledgement}
\vspace{-1mm}

This study is supported in part by NSF Award NeTS-2312675.

\newpage
\bibliography{Reference}

\begin{thebibliography}{40}
\providecommand{\natexlab}[1]{#1}

\bibitem[{Amini et~al.(2019)Amini, Gabriel, Lin, Koncel{-}Kedziorski, Choi, and Hajishirzi}]{DBLP:conf/naacl/AminiGLKCH19}
Aida Amini, Saadia Gabriel, Shanchuan Lin, Rik Koncel{-}Kedziorski, Yejin Choi, and Hannaneh Hajishirzi. 2019.
\newblock Mathqa: Towards interpretable math word problem solving with operation-based formalisms.
\newblock In \emph{{NAACL-HLT} {(1)}}, pages 2357--2367. Association for Computational Linguistics.

\bibitem[{Ashkboos et~al.(2024)Ashkboos, Croci, Nascimento, Hoefler, and Hensman}]{DBLP:conf/iclr/AshkboosCNHH24}
Saleh Ashkboos, Maximilian~L. Croci, Marcelo Gennari~Do Nascimento, Torsten Hoefler, and James Hensman. 2024.
\newblock Slicegpt: Compress large language models by deleting rows and columns.
\newblock In \emph{{ICLR}}. OpenReview.net.

\bibitem[{Bisk et~al.(2020)Bisk, Zellers, Bras, Gao, and Choi}]{DBLP:conf/aaai/BiskZLGC20}
Yonatan Bisk, Rowan Zellers, Ronan~Le Bras, Jianfeng Gao, and Yejin Choi. 2020.
\newblock {PIQA:} reasoning about physical commonsense in natural language.
\newblock In \emph{{AAAI}}, pages 7432--7439. {AAAI} Press.

\bibitem[{Clark et~al.(2018)Clark, Cowhey, Etzioni, Khot, Sabharwal, Schoenick, and Tafjord}]{DBLP:journals/corr/abs-1803-05457}
Peter Clark, Isaac Cowhey, Oren Etzioni, Tushar Khot, Ashish Sabharwal, Carissa Schoenick, and Oyvind Tafjord. 2018.
\newblock Think you have solved question answering? try arc, the {AI2} reasoning challenge.
\newblock \emph{CoRR}, abs/1803.05457.

\bibitem[{Cobbe et~al.(2021)Cobbe, Kosaraju, Bavarian, Chen, Jun, Kaiser, Plappert, Tworek, Hilton, Nakano, Hesse, and Schulman}]{DBLP:journals/corr/abs-2110-14168}
Karl Cobbe, Vineet Kosaraju, Mohammad Bavarian, Mark Chen, Heewoo Jun, Lukasz Kaiser, Matthias Plappert, Jerry Tworek, Jacob Hilton, Reiichiro Nakano, Christopher Hesse, and John Schulman. 2021.
\newblock Training verifiers to solve math word problems.
\newblock \emph{CoRR}, abs/2110.14168.

\bibitem[{Frantar and Alistarh(2023)}]{DBLP:conf/icml/FrantarA23}
Elias Frantar and Dan Alistarh. 2023.
\newblock Sparsegpt: Massive language models can be accurately pruned in one-shot.
\newblock In \emph{{ICML}}, volume 202 of \emph{Proceedings of Machine Learning Research}, pages 10323--10337. {PMLR}.

\bibitem[{Frantar et~al.(2022)Frantar, Ashkboos, Hoefler, and Alistarh}]{DBLP:journals/corr/abs-2210-17323}
Elias Frantar, Saleh Ashkboos, Torsten Hoefler, and Dan Alistarh. 2022.
\newblock {GPTQ:} accurate post-training quantization for generative pre-trained transformers.
\newblock \emph{CoRR}, abs/2210.17323.

\bibitem[{Gao et~al.(2023)Gao, Tow, Abbasi, Biderman, Black, DiPofi, Foster, Golding, Hsu, Le~Noac'h, Li, McDonell, Muennighoff, Ociepa, Phang, Reynolds, Schoelkopf, Skowron, Sutawika, Tang, Thite, Wang, Wang, and Zou}]{eval-harness}
Leo Gao, Jonathan Tow, Baber Abbasi, Stella Biderman, Sid Black, Anthony DiPofi, Charles Foster, Laurence Golding, Jeffrey Hsu, Alain Le~Noac'h, Haonan Li, Kyle McDonell, Niklas Muennighoff, Chris Ociepa, Jason Phang, Laria Reynolds, Hailey Schoelkopf, Aviya Skowron, Lintang Sutawika, Eric Tang, Anish Thite, Ben Wang, Kevin Wang, and Andy Zou. 2023.
\newblock \href {https://doi.org/10.5281/zenodo.10256836} {A framework for few-shot language model evaluation}.

\bibitem[{Golub et~al.(1987)Golub, Hoffman, and Stewart}]{GOLUB1987317}
G.H. Golub, Alan Hoffman, and G.W. Stewart. 1987.
\newblock \href {https://doi.org/10.1016/0024-3795(87)90114-5} {A generalization of the eckart-young-mirsky matrix approximation theorem}.
\newblock \emph{Linear Algebra and its Applications}, 88-89:317--327.

\bibitem[{Gozalo{-}Brizuela and Garrido{-}Merch{\'{a}}n(2023)}]{DBLP:journals/corr/abs-2306-02781}
Roberto Gozalo{-}Brizuela and Eduardo~C. Garrido{-}Merch{\'{a}}n. 2023.
\newblock A survey of generative {AI} applications.
\newblock \emph{CoRR}, abs/2306.02781.

\bibitem[{He et~al.(2024)He, Sun, Shen, and Li}]{DBLP:journals/corr/abs-2406-15786}
Shwai He, Guoheng Sun, Zheyu Shen, and Ang Li. 2024.
\newblock What matters in transformers? not all attention is needed.
\newblock \emph{CoRR}, abs/2406.15786.

\bibitem[{Hsu et~al.(2022)Hsu, Hua, Chang, Lou, Shen, and Jin}]{DBLP:conf/iclr/HsuHCLSJ22}
Yen{-}Chang Hsu, Ting Hua, Sungen Chang, Qian Lou, Yilin Shen, and Hongxia Jin. 2022.
\newblock Language model compression with weighted low-rank factorization.
\newblock In \emph{{ICLR}}. OpenReview.net.

\bibitem[{Huang et~al.(2024)Huang, Liu, Qin, Li, Zhang, Liu, Magno, and Qi}]{DBLP:conf/icml/HuangLQLZ0M024}
Wei Huang, Yangdong Liu, Haotong Qin, Ying Li, Shiming Zhang, Xianglong Liu, Michele Magno, and Xiaojuan Qi. 2024.
\newblock Billm: Pushing the limit of post-training quantization for llms.
\newblock In \emph{{ICML}}. OpenReview.net.

\bibitem[{Ji et~al.(2024)Ji, Xiang, Li, Chen, Liu, Chen, and Zhang}]{DBLP:journals/corr/abs-2405-10616}
Yixin Ji, Yang Xiang, Juntao Li, Wei Chen, Zhongyi Liu, Kehai Chen, and Min Zhang. 2024.
\newblock Feature-based low-rank compression of large language models via bayesian optimization.
\newblock \emph{CoRR}, abs/2405.10616.

\bibitem[{Li et~al.(2024)Li, Dong, and Lei}]{DBLP:journals/corr/abs-2407-19126}
Jianwei Li, Yijun Dong, and Qi~Lei. 2024.
\newblock Greedy output approximation: Towards efficient structured pruning for llms without retraining.
\newblock \emph{CoRR}, abs/2407.19126.

\bibitem[{Lin et~al.(2024{\natexlab{a}})Lin, Gao, Smith, Patel, Tuli, Shen, Jin, and Hsu}]{DBLP:journals/corr/abs-2408-09632}
Chi{-}Heng Lin, Shangqian Gao, James~Seale Smith, Abhishek Patel, Shikhar Tuli, Yilin Shen, Hongxia Jin, and Yen{-}Chang Hsu. 2024{\natexlab{a}}.
\newblock Modegpt: Modular decomposition for large language model compression.
\newblock \emph{CoRR}, abs/2408.09632.

\bibitem[{Lin et~al.(2022)Lin, Hilton, and Evans}]{DBLP:conf/acl/LinHE22}
Stephanie Lin, Jacob Hilton, and Owain Evans. 2022.
\newblock Truthfulqa: Measuring how models mimic human falsehoods.
\newblock In \emph{{ACL} {(1)}}, pages 3214--3252. Association for Computational Linguistics.

\bibitem[{Lin et~al.(2024{\natexlab{b}})Lin, Tang, Yang, Zhang, Xiao, Gan, and Han}]{DBLP:journals/corr/abs-2405-04532}
Yujun Lin, Haotian Tang, Shang Yang, Zhekai Zhang, Guangxuan Xiao, Chuang Gan, and Song Han. 2024{\natexlab{b}}.
\newblock Qserve: {W4A8KV4} quantization and system co-design for efficient {LLM} serving.
\newblock \emph{CoRR}, abs/2405.04532.

\bibitem[{Ma et~al.(2023)Ma, Fang, and Wang}]{DBLP:conf/nips/MaFW23}
Xinyin Ma, Gongfan Fang, and Xinchao Wang. 2023.
\newblock Llm-pruner: On the structural pruning of large language models.
\newblock In \emph{NeurIPS}.

\bibitem[{Merity et~al.(2017)Merity, Xiong, Bradbury, and Socher}]{DBLP:conf/iclr/MerityX0S17}
Stephen Merity, Caiming Xiong, James Bradbury, and Richard Socher. 2017.
\newblock Pointer sentinel mixture models.
\newblock In \emph{{ICLR} (Poster)}. OpenReview.net.

\bibitem[{Mihaylov et~al.(2018)Mihaylov, Clark, Khot, and Sabharwal}]{DBLP:conf/emnlp/MihaylovCKS18}
Todor Mihaylov, Peter Clark, Tushar Khot, and Ashish Sabharwal. 2018.
\newblock Can a suit of armor conduct electricity? {A} new dataset for open book question answering.
\newblock In \emph{{EMNLP}}, pages 2381--2391. Association for Computational Linguistics.

\bibitem[{Raffel et~al.(2020)Raffel, Shazeer, Roberts, Lee, Narang, Matena, Zhou, Li, and Liu}]{DBLP:journals/jmlr/RaffelSRLNMZLL20}
Colin Raffel, Noam Shazeer, Adam Roberts, Katherine Lee, Sharan Narang, Michael Matena, Yanqi Zhou, Wei Li, and Peter~J. Liu. 2020.
\newblock Exploring the limits of transfer learning with a unified text-to-text transformer.
\newblock \emph{J. Mach. Learn. Res.}, 21:140:1--140:67.

\bibitem[{Sakaguchi et~al.(2020)Sakaguchi, Bras, Bhagavatula, and Choi}]{DBLP:conf/aaai/SakaguchiBBC20}
Keisuke Sakaguchi, Ronan~Le Bras, Chandra Bhagavatula, and Yejin Choi. 2020.
\newblock Winogrande: An adversarial winograd schema challenge at scale.
\newblock In \emph{{AAAI}}, pages 8732--8740. {AAAI} Press.

\bibitem[{Shen et~al.(2024)Shen, Wan, Wang, and Zhang}]{DBLP:journals/corr/abs-2409-09808}
Hui Shen, Zhongwei Wan, Xin Wang, and Mi~Zhang. 2024.
\newblock Famba-v: Fast vision mamba with cross-layer token fusion.
\newblock \emph{CoRR}, abs/2409.09808.

\bibitem[{Shen et~al.(2025)Shen, Zhang, Xiong, Hu, Chen, Wan, Wang, Zhang, Gong, Bao et~al.}]{shen2025efficient}
Hui Shen, Jingxuan Zhang, Boning Xiong, Rui Hu, Shoufa Chen, Zhongwei Wan, Xin Wang, Yu~Zhang, Zixuan Gong, Guangyin Bao, et~al. 2025.
\newblock Efficient diffusion models: A survey.
\newblock \emph{arXiv preprint arXiv:2502.06805}.

\bibitem[{Touvron et~al.(2023)Touvron, Martin, Stone, Albert, Almahairi, Babaei, Bashlykov, Batra, Bhargava, Bhosale, Bikel, Blecher, Canton{-}Ferrer, Chen, Cucurull, Esiobu, Fernandes, Fu, Fu, Fuller, Gao, Goswami, Goyal, Hartshorn, Hosseini, Hou, Inan, Kardas, Kerkez, Khabsa, Kloumann, Korenev, Koura, Lachaux, Lavril, Lee, Liskovich, Lu, Mao, Martinet, Mihaylov, Mishra, Molybog, Nie, Poulton, Reizenstein, Rungta, Saladi, Schelten, Silva, Smith, Subramanian, Tan, Tang, Taylor, Williams, Kuan, Xu, Yan, Zarov, Zhang, Fan, Kambadur, Narang, Rodriguez, Stojnic, Edunov, and Scialom}]{DBLP:journals/corr/abs-2307-09288}
Hugo Touvron, Louis Martin, Kevin Stone, Peter Albert, Amjad Almahairi, Yasmine Babaei, Nikolay Bashlykov, Soumya Batra, Prajjwal Bhargava, Shruti Bhosale, Dan Bikel, Lukas Blecher, Cristian Canton{-}Ferrer, Moya Chen, Guillem Cucurull, David Esiobu, Jude Fernandes, Jeremy Fu, Wenyin Fu, Brian Fuller, Cynthia Gao, Vedanuj Goswami, Naman Goyal, Anthony Hartshorn, Saghar Hosseini, Rui Hou, Hakan Inan, Marcin Kardas, Viktor Kerkez, Madian Khabsa, Isabel Kloumann, Artem Korenev, Punit~Singh Koura, Marie{-}Anne Lachaux, Thibaut Lavril, Jenya Lee, Diana Liskovich, Yinghai Lu, Yuning Mao, Xavier Martinet, Todor Mihaylov, Pushkar Mishra, Igor Molybog, Yixin Nie, Andrew Poulton, Jeremy Reizenstein, Rashi Rungta, Kalyan Saladi, Alan Schelten, Ruan Silva, Eric~Michael Smith, Ranjan Subramanian, Xiaoqing~Ellen Tan, Binh Tang, Ross Taylor, Adina Williams, Jian~Xiang Kuan, Puxin Xu, Zheng Yan, Iliyan Zarov, Yuchen Zhang, Angela Fan, Melanie Kambadur, Sharan Narang, Aur{\'{e}}lien Rodriguez, Robert Stojnic, Sergey Edunov,
  and Thomas Scialom. 2023.
\newblock Llama 2: Open foundation and fine-tuned chat models.
\newblock \emph{CoRR}, abs/2307.09288.

\bibitem[{Wan et~al.(2025)Wan, Shen, Wang, Liu, Mai, and Zhang}]{wan2025meda}
Zhongwei Wan, Hui Shen, Xin Wang, Che Liu, Zheda Mai, and Mi~Zhang. 2025.
\newblock Meda: Dynamic kv cache allocation for efficient multimodal long-context inference.
\newblock \emph{arXiv preprint arXiv:2502.17599}.

\bibitem[{Wan et~al.(2024{\natexlab{a}})Wan, Wang, Liu, Alam, Zheng, Liu, Qu, Yan, Zhu, Zhang, Chowdhury, and Zhang}]{DBLP:journals/tmlr/Wan0LA0LQYZZC024}
Zhongwei Wan, Xin Wang, Che Liu, Samiul Alam, Yu~Zheng, Jiachen Liu, Zhongnan Qu, Shen Yan, Yi~Zhu, Quanlu Zhang, Mosharaf Chowdhury, and Mi~Zhang. 2024{\natexlab{a}}.
\newblock Efficient large language models: {A} survey.
\newblock \emph{Trans. Mach. Learn. Res.}, 2024.

\bibitem[{Wan et~al.(2024{\natexlab{b}})Wan, Wu, Zhang, Xin, Tao, Zhu, Wang, Luo, Xiong, and Zhang}]{DBLP:journals/corr/abs-2406-13035}
Zhongwei Wan, Xinjian Wu, Yu~Zhang, Yi~Xin, Chaofan Tao, Zhihong Zhu, Xin Wang, Siqi Luo, Jing Xiong, and Mi~Zhang. 2024{\natexlab{b}}.
\newblock {D2O:} dynamic discriminative operations for efficient generative inference of large language models.
\newblock \emph{CoRR}, abs/2406.13035.

\bibitem[{Wang et~al.(2024{\natexlab{a}})Wang, Wan, Hekmati, Zong, Alam, Zhang, and Krishnamachari}]{DBLP:journals/internet/WangWHZAZK24}
Xin Wang, Zhongwei Wan, Arvin Hekmati, Mingyu Zong, Samiul Alam, Mi~Zhang, and Bhaskar Krishnamachari. 2024{\natexlab{a}}.
\newblock The internet of things in the era of generative {AI:} vision and challenges.
\newblock \emph{{IEEE} Internet Comput.}, 28(5):57--64.

\bibitem[{Wang et~al.(2024{\natexlab{b}})Wang, Zheng, Wan, and Zhang}]{DBLP:journals/corr/abs-2403-07378}
Xin Wang, Yu~Zheng, Zhongwei Wan, and Mi~Zhang. 2024{\natexlab{b}}.
\newblock {SVD-LLM:} truncation-aware singular value decomposition for large language model compression.
\newblock \emph{CoRR}, abs/2403.07378.

\bibitem[{Yuan et~al.(2024)Yuan, Shang, and Dong}]{DBLP:conf/iclr/YuanSD24}
Zhihang Yuan, Yuzhang Shang, and Zhen Dong. 2024.
\newblock {PB-LLM:} partially binarized large language models.
\newblock In \emph{{ICLR}}. OpenReview.net.

\bibitem[{Yuan et~al.(2023)Yuan, Shang, Song, Wu, Yan, and Sun}]{DBLP:journals/corr/abs-2312-05821}
Zhihang Yuan, Yuzhang Shang, Yue Song, Qiang Wu, Yan Yan, and Guangyu Sun. 2023.
\newblock {ASVD:} activation-aware singular value decomposition for compressing large language models.
\newblock \emph{CoRR}, abs/2312.05821.

\bibitem[{Zellers et~al.(2019)Zellers, Holtzman, Bisk, Farhadi, and Choi}]{DBLP:conf/acl/ZellersHBFC19}
Rowan Zellers, Ari Holtzman, Yonatan Bisk, Ali Farhadi, and Yejin Choi. 2019.
\newblock Hellaswag: Can a machine really finish your sentence?
\newblock In \emph{{ACL} {(1)}}, pages 4791--4800. Association for Computational Linguistics.

\bibitem[{Zhang et~al.(2022)Zhang, Roller, Goyal, Artetxe, Chen, Chen, Dewan, Diab, Li, Lin, Mihaylov, Ott, Shleifer, Shuster, Simig, Koura, Sridhar, Wang, and Zettlemoyer}]{DBLP:journals/corr/abs-2205-01068}
Susan Zhang, Stephen Roller, Naman Goyal, Mikel Artetxe, Moya Chen, Shuohui Chen, Christopher Dewan, Mona~T. Diab, Xian Li, Xi~Victoria Lin, Todor Mihaylov, Myle Ott, Sam Shleifer, Kurt Shuster, Daniel Simig, Punit~Singh Koura, Anjali Sridhar, Tianlu Wang, and Luke Zettlemoyer. 2022.
\newblock {OPT:} open pre-trained transformer language models.
\newblock \emph{CoRR}, abs/2205.01068.

\bibitem[{Zhao et~al.(2023)Zhao, Zhou, Li, Tang, Wang, Hou, Min, Zhang, Zhang, Dong, Du, Yang, Chen, Chen, Jiang, Ren, Li, Tang, Liu, Liu, Nie, and Wen}]{DBLP:journals/corr/abs-2303-18223}
Wayne~Xin Zhao, Kun Zhou, Junyi Li, Tianyi Tang, Xiaolei Wang, Yupeng Hou, Yingqian Min, Beichen Zhang, Junjie Zhang, Zican Dong, Yifan Du, Chen Yang, Yushuo Chen, Zhipeng Chen, Jinhao Jiang, Ruiyang Ren, Yifan Li, Xinyu Tang, Zikang Liu, Peiyu Liu, Jian{-}Yun Nie, and Ji{-}Rong Wen. 2023.
\newblock A survey of large language models.
\newblock \emph{CoRR}, abs/2303.18223.

\bibitem[{Zhao et~al.(2024)Zhao, Shi, Lyu, Sui, Li, and Li}]{DBLP:journals/corr/abs-2411-07762}
Weibo Zhao, Yubin Shi, Xinyu Lyu, Wanchen Sui, Shen Li, and Yong Li. 2024.
\newblock {ASER:} activation smoothing and error reconstruction for large language model quantization.
\newblock \emph{CoRR}, abs/2411.07762.

\bibitem[{Zhong et~al.(2024)Zhong, Wan, Chen, Quan, and Li}]{DBLP:journals/corr/abs-2406-10594}
Longguang Zhong, Fanqi Wan, Ruijun Chen, Xiaojun Quan, and Liangzhi Li. 2024.
\newblock Blockpruner: Fine-grained pruning for large language models.
\newblock \emph{CoRR}, abs/2406.10594.

\bibitem[{Zhou et~al.(2024)Zhou, Ning, Hong, Fu, Xu, Li, Lou, Wang, Yuan, Li, Yan, Dai, Zhang, Dong, and Wang}]{zhou2024survey}
Zixuan Zhou, Xuefei Ning, Ke~Hong, Tianyu Fu, Jiaming Xu, Shiyao Li, Yuming Lou, Luning Wang, Zhihang Yuan, Xiuhong Li, Shengen Yan, Guohao Dai, Xiao-Ping Zhang, Yuhan Dong, and Yu~Wang. 2024.
\newblock \href {https://arxiv.org/abs/2404.14294} {A survey on efficient inference for large language models}.
\newblock \emph{Preprint}, arXiv:2404.14294.

\bibitem[{Zhu et~al.(2023)Zhu, Li, Liu, Ma, and Wang}]{DBLP:journals/corr/abs-2308-07633}
Xunyu Zhu, Jian Li, Yong Liu, Can Ma, and Weiping Wang. 2023.
\newblock A survey on model compression for large language models.
\newblock \emph{CoRR}, abs/2308.07633.

\end{thebibliography}
\end{document}